\documentclass{article}

\usepackage{microtype}
\usepackage{graphicx}
\usepackage{booktabs} 

\usepackage{hyperref}

\usepackage{tikz}
\usetikzlibrary{positioning, calc}
\usetikzlibrary{shapes.geometric}
\usepackage{amsmath}
\usepackage{amssymb}
\usepackage{pgfplots}
\usepackage{mathtools}
\usepackage[inline]{enumitem}
\usepackage{caption}
\usepackage{subcaption}
\usepackage{threeparttable} 
\usepackage{glossaries}
\usepackage{algorithm}
\usepackage{algorithmic}
\usepackage{cleveref}
\usepackage{amsthm}

\DeclareMathOperator*{\argmin}{argmin}
\newtheorem{theorem}{Theorem}

\setacronymstyle{long-short}

\newacronym{lnt}{LNT}{Local Neural Transformations}
\newacronym{ad}{AD}{Anomaly detection}
\newacronym{ddcl}{DDCL}{dynamic deterministic contrastive loss}
\newacronym{swat}{SWaT}{Secure Water Treatment Dataset}
\newacronym{wadi}{WaDi}{Water Distribution Dataset}
\newacronym{neutral}{NeuTraL}{Neural Transformation Learning}
\newacronym{cpc}{CPC}{Contrastive Predictive Coding}
\newacronym{dcl}{DCL}{Deterministic Contrastive Loss}
\newacronym{ocsvm}{OC-SVM}{OC-SVM}
\newacronym{lof}{LOF}{Local Outlier Factor}
\newacronym{gdn}{GDN}{Graph Deviation Network}
\newacronym{lstm}{LSTM}{Long Short Term Memory}
\usepackage[arxiv]{icml2021}

\icmltitlerunning{Detecting Anomalies within Time Series using Local Neural Transformations}

\begin{document}

\twocolumn[
\icmltitle{Detecting Anomalies within Time Series using Local Neural Transformations}



\icmlsetsymbol{equal}{*}

\begin{icmlauthorlist}
\icmlauthor{Tim Schneider}{equal,ust}
\icmlauthor{Chen Qiu}{equal,bcai,tuk}
\icmlauthor{Marius Kloft}{tuk}
\icmlauthor{Decky Aspandi Latif}{ust}
\icmlauthor{Steffen Staab}{ust,uso}
\icmlauthor{Stephan Mandt}{uci}
\icmlauthor{Maja Rudolph}{bcai}
\end{icmlauthorlist}

\icmlaffiliation{bcai}{Bosch Center for Artificial Intelligence}
\icmlaffiliation{tuk}{TU Kaiserslautern, Germany}
\icmlaffiliation{ust}{University of Stuttgart, Germany}
\icmlaffiliation{uci}{UC Irvine, USA}
\icmlaffiliation{uso}{University of Southampton, UK}

\icmlcorrespondingauthor{Tim Schneider}{timphillip.schneider@ipvs.uni-stuttgart.de}
\icmlcorrespondingauthor{Maja Rudolph}{Maja.Rudolph@us.bosch.com}

\icmlkeywords{Machine Learning, ICML}

\vskip 0.3in
]



\printAffiliationsAndNotice{\icmlEqualContribution} 

\begin{abstract}
We develop a new method to detect anomalies within time series, which is essential in many application domains, reaching from self-driving cars, finance, and marketing to medical diagnosis and epidemiology. The method is based on self-supervised deep learning that has played a key role in facilitating deep anomaly detection on images, where powerful image transformations are available. However, such transformations are widely unavailable for time series. Addressing this, we develop \gls{lnt}, a method learning local transformations of time series from data. The method produces an anomaly score for each time step and thus can be used to detect anomalies within time series. We prove in a theoretical analysis that our novel training objective is more suitable for transformation learning than previous deep \gls{ad} methods. Our experiments demonstrate that \gls{lnt} can find anomalies in speech segments from the LibriSpeech data set and better detect interruptions to cyber-physical systems than previous work. Visualization of the learned transformations gives insight into the type of transformations that \gls{lnt} learns. 
\end{abstract}
\glsresetall
\section{Introduction}
\Gls{ad} in time series is significant in many industrial, medical, and scientific applications. For instance, undetected anomalies in water treatment facilities or chemical plants can bring harm to millions of people. Such systems need to be constantly monitored for anomalies.

While AD has been an important field in machine learning for several decades~\citep{ruff2020unifying}, promising performance gains have been primarily reported in applying deep learning methods to high-dimensional data such as images \citep{golan2018deep,wang2019effective,hendrycks2019using,bergman2020classification}. With few exceptions~\citep{gidaris2018unsupervised,hundman2018detecting,shen2020timeseries}, there is much less work on other domains such as time series. 
This may be attributed to the fact that some time series exhibit complex temporal dependencies and can be even more diverse than natural images. As detailed below, this paper attempts to integrate recent ideas from self-supervised \gls{ad} of non-temporal data with modern deep learning architectures for sequence modeling.  

While \emph{unsupervised} methods based on density estimation can yield poor results for \gls{ad} \citep{nalisnick2018deep}, a recent trend relying on \emph{self-supervision} has proven superior performance. 
In this line of work, one uses auxiliary tasks, often based on data augmentation, both for training and anomaly scoring. 
Data augmentation usually relies on hand-designed data transformations such as rotations \citep{golan2018deep,wang2019effective,hendrycks2019using}. \citet{qiu2021neural} showed that these transformations could instead be \emph{learned}, thereby making self-supervised \gls{ad} applicable to specialized domains beyond images. 
While this approach can identify an entire sequence as anomalous, it can still not be applied to detecting anomalies \emph{within} time series (i.e., on a sub-sequence level).


But this adaption is not straightforward: For AD within time series, both local semantics (the dynamics within a time window) and contextualized semantics (how the time window relates to the remaining time series) matter. To capture both, 
we propose an end-to-end approach that combines 
time series representations~\citep{oord2018representation} with a novel transformation learning objective. As a result, the local transformations create different views of the data in the latent space \citep{rudolph2017structured} (as opposed to applying them to the data directly as in \citet{qiu2021neural}).  

We develop \gls{lnt}: a 
novel objective
that combines representation learning with transformation learning.
The encoder for feature extraction and the neural transformations are trained jointly on this loss. We show that the learned latent transformations can correspond to interpretable effects: in one experiment on speech data (details in \Cref{sec:experiments}), \gls{lnt} learns transformations that insert delays. Neural transformations are much more general than hand-crafted transformations, which for time series could be time warping, reflections, or shifts: as we illustrate, they can transform the data in ways unintuitive to humans but valuable for the downstream task of \gls{ad}. 

We prove theoretically (\Cref{sec:theory}) and show empirically (\Cref{sec:experiments}) that combining representation and transformation learning is beneficial for detecting anomalies within time series. \gls{lnt} outperforms various \gls{ad} techniques on benchmark data, including a baseline using the \gls{cpc} loss as the anomaly score \citep{de2021contrastive}. We evaluate the methods on public \gls{ad} datasets for time series from cyber-physical systems. Furthermore, we detect artificial anomalies in speech data, which is challenging due to its complex temporal dynamics. In all experiments, \gls{lnt} outperforms many strong baselines.

To summarize, our contributions in this work are:
\begin{enumerate}
    \item A {\em new method}, \gls{lnt}, for \gls{ad} within time series. It unifies time series representations with a novel approach for learning {\em local} transformations. A open-source pytorch implementation is available at \url{https://github.com/boschresearch/local_neural_transformations}.
    \item A {\em theoretical analysis}. We prove that both learning paradigms complement each other to avoid trivial solutions not appropriate for detecting anomalies. 
    \item An extensive {\em empirical study} showing that \gls{lnt} can detect anomalies within real cyber-physical data streams on par or better than many existing methods. 
\end{enumerate}
\section{Related Work}
We first describe related work in time series \gls{ad}, which is the problem we tackle in this work. We then describe related methods, specifically advances in self-supervised \gls{ad}.

\subsection{Time series anomaly detection}
There are two types of anomalies in time series: local and global anomalies. Global anomalies are entire time series, with a single anomaly score for the entire series. Local anomalies occur at isolated timestamps or short time intervals within the time series, so each time point must be assigned with an anomaly score. This is the setting that we consider in this work. 
Existing methods for local \gls{ad} in time series using deep learning can be divided into four categories, discussed in detail below:
\begin{enumerate*}[label=(\roman*)]
    \item methods based on sequence forecasting,
    \item autoencoders,
    \item generative sequence models, and
    \item other approaches.
\end{enumerate*}

\paragraph{Forecasting methods}
A straightforward approach to detect anomalies in time series is to use the error of a time-series forecaster (predicting the value of the next time step from the time series' past history) as an anomaly score. The rationale behind is that a forecaster trained on mostly normal data will err less on normal than on abnormal data. We may use any time-series regression method as the forecaster, and various methods have been studied, including neural architectures such as recurrent neural networks (RNNs) \citep{malhotra2015long, filonov2016multivariate} and temporal convolutional neural networks (TCNs) \citep{he2019temporal, munir2019deepant}, where the convolution operation is applied along the temporal dimension only. 

\paragraph{Autoencoders}
To detect anomalies within time series, AEs have been combined with various neural network architectures, including RNNs \citep{malhotra2016lstm} and TCNs \citep{thill2020time} or variants \citep{zhang2019deep}. \citet{audibert2020usad} propose an architecture based purely on dense layers using a combination of two AEs connected with the adversarial loss. Again, the rational of using such approaches for \gls{ad} is that after training on normal data, a high reconstruction error can be used to detect anomalies.

\paragraph{Deep generative models}
Variational autoencoders (VAEs) \citep{kingma2014auto} have frequently been combined with RNNs \citep{solch2016variational, park2018multimodal} to detect anomalies within time series. \citet{pereira2018unsupervised} combine an RNN with temporal self-attention. \citet{guo2018multidimensional} use gated recurrent units (GRUs) in combination with a gaussian mixture model. \citet{su2019robust} augment a GRU-based VAE with a normalizing flow and a linear Gaussian state-space model. Generative adversarial networks \citep{goodfellow2014generative} have been used for AD within time series, taking either the discriminator's error \citep{liang2021robust} or the generator's residuals \citep{zhou2019beatgan} as an anomaly score. \citet{li2019mad} use a weighted combination of both. These approaches have been combined with TCNs \citep{zhou2019beatgan} and RNNs \citep{niu2020lstm,geiger2020tadgan}.

\paragraph{Other methods}
Some of the above-described approaches have been used in combination. For instance, \citet{zhao2020multivariate} combine TCNs and LSTMs. \citet{shen2020timeseries} combine a dilated RNN with a deep multi-sphere hypersphere classifier on the cluster centers of a hierarchical clustering procedure, with regularizers encouraging orthogonal centers at each layer and prediction regularizers encouraging useful representations in intermediate layers. \citet{deng2021graph} construct a graph with nodes for each feature and edges representing relations between features; these are learned and combined with a graph-based attention mechanism. \citet{carmona2021neural} employ a TCN as an encoder to train a hypersphere classifier in the latent space, with the option of including known anomalies into training.



\subsection{Self-supervised anomaly detection}
Recently, there has been growing interest in tackling \gls{ad} with \textit{self-supervised learning}. 
The core idea of self-supervised learning is to devise training tasks, often based on data augmentation, that guide the model to learn useful representations of the data. In self-supervised \gls{ad}, performance on the auxiliary tasks can be used for anomaly scoring. This is justified by the principle of inlier priority  \citep{wang2019effective} which posits that a self-supervised approach will prioritize solving its training task for inliers. 
 End-to-end detection methods based on transformation prediction \citep{golan2018deep,hendrycks2019using} have been designed for image \gls{ad}. 
However, they require effective hand-crafted transformations while for data types beyond images, it is hard to design effective transformations by hand. Previous works proposed to utilize random affine transformations \citep{bergman2020classification} or data-driven neural transformations \citep{qiu2021neural} for \gls{ad}.
Neural transformations have been used to detect entire anomalous sequences. However, when the neural transformation learning approach of \citet{qiu2021neural} is applied to the task of local anomaly detection, it can lead to trivial transformations that are not suitable for \gls{ad}. Our work proves this and introduces a novel {\em local} transformation learning objective. 

Alternatively, \citet{de2021contrastive} propose to use the training criterion of \gls{cpc}, a self-supervised approach without data augmentation, for anomaly detection. \Gls{cpc} learns local time series representations via contrastive predictions of future representations \citep{oord2018representation}. However, the \gls{cpc} loss is not a good fit for scoring anomalies since it requires a random draw of negative samples, which leads to a biased estimation or high memory cost during test time \citep{de2021contrastive}.
Our work overcomes this.
\section{Method}
\label{method}


In this work, we propose \textit{Local Neural Transformations (LNT)}, a new framework for detecting anomalies within time series data. \gls{lnt} has two components: feature extraction and feature transformations. Given an input sequence, an encoder produces an embedding for each time step, encoding relevant information from the current time window. These features are then transformed by applying distinct neural networks to each embedding, producing different {\em latent views}. The views are trained to fulfill two requirements; the views should be diverse and semantically meaningful, i.e., they should reflect both local dynamics as well as how the observations fit into the larger context of the time series.
The requirements are encouraged via self-supervision.

Specifically, two aspects of \gls{lnt} are \textit{self-supervised} -- it combines two different contrastive losses. One of the contrastive losses, \gls{cpc}, guides the representation learning that guarantees the encoder of \gls{lnt} to produce good semantic time series representations that generalize well to unseen test data. The second contrastive loss, a novel \gls{ddcl}, contrasts different latent views of each time step to encourage the latent views to be diverse and semantically representative of the time series, both in a local and in a contextualized sense.

\gls{lnt} follows the general paradigm of self-supervised \gls{ad}. During training, the capability to contrast the data views produced by the transformations improves for the normal data, while it deteriorates for anomalies. 
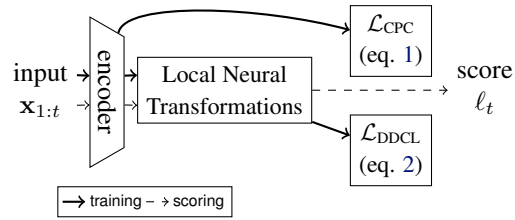
\begin{figure}[t]
    \centering
    \begin{tikzpicture}
    \node[align=center] (input) {input \\ $\mathbf{x}_{1:t}$};
    \node[draw, trapezium, inner sep=2pt, rotate=270, right= 4mm of input, xshift=-8.5mm, minimum width= 20mm] (enc) {encoder};
    
    \node[draw, align=center, right= 8mm of input] (lnt) {\small{Local Neural} \\ \small{Transformations}};
    
    \node[draw, align=center, right= 5mm of lnt, yshift=-8mm] (tr_loss) {\small{$\mathcal{L}_\text{DDCL}$} \\ \small{(eq. \ref{eq:ddcl_loss})}};
    
    \node[draw, align=center, above= 5mm of tr_loss] (emb_loss) {\small{$\mathcal{L}_\text{CPC}$} \\ \small{(eq. \ref{eq:cpc_loss})}};
    
    \node[right= 18mm of lnt, align=center] (score) {score \\ $\ell_t$};
    
    \draw[->, thick] ($(input.east)+(0.0,2mm)$) -- ($(enc.south)+(0.0,2.2mm)$);
    \draw[->, dashed] ($(input.east)+(0.0,-2mm)$) -- ($(enc.south)+(0.0,-1.7mm)$);
    
    \draw[->, thick] ($(enc.north)+(0.0,2.2mm)$) -- ($(lnt.west)+(0.0,2mm)$);
    \draw[->, dashed] ($(enc.north)+(0.0,-1.7mm)$) -- ($(lnt.west)+(0.0,-2mm)$);
    
    \draw[->, thick] (enc) to [in=165, out=90] (emb_loss);
    \draw[->, thick] (lnt) -- (tr_loss);
    \draw[->, dashed] (lnt) -- (score);
    
    \path (current bounding box.south west)
     node[matrix,anchor=north west,cells={nodes={font=\sffamily,anchor=west}},
     draw,inner sep=0.5mm, xshift= 2em]{
      \draw[->, thick](0,0) -- ++ (0.3,0); & \node{\tiny training}; &
      \draw[->, dashed](0,0) -- ++ (0.3,0); & \node{\tiny scoring};\\
     };
    
\end{tikzpicture}
    \caption{Overview of Local Neural Transformations (LNT): neural transformations are jointly trained with the encoder from \gls{cpc} and \gls{ddcl} losses; at test time yielding anomaly scores $\ell_t$ on a sub-sequence level given a series $x_{1:t}$.}
    \label{fig:method_overview}
\end{figure}
\Cref{fig:method_overview} summarizes the main components of \gls{lnt}. Given a (potentially multivariate) time series 
$x_{1:t} \coloneqq (x_1, \ldots, x_t)^T \; ; x_t \in \mathbb{R}^d$, 
our method should output scores $\ell_{t}$ for each individual time step, representing the likelihood that the observation in this time step is an anomaly. 
The inputs are processed by an encoder that is trained jointly with the local neural transformations.

Before presenting local transformation learning and the \gls{ddcl} in \Cref{sec:transformation_learning}, we will first describe the encoder and the \gls{cpc}-loss in \Cref{sec:representation_learning}.
Then, we discuss how a trained model is used to detect anomalies. Finally, in \Cref{sec:theory}, we provide theoretical arguments for combining transformation learning with representation learning.

\subsection{Local Time Series Representations}
\label{sec:representation_learning}
The \gls{lnt} architecture has two components, a feature extractor (encoder) and an anomaly detector (local neural transformations).
The encoder maps a sequence of samples to a sequence of local latent representations $z_t$ and is trained using the principles of \emph{Contrastive Predictive Coding (CPC)} \citep{oord2018representation}. We use the same architecture as \citet{oord2018representation}. The representations produced by the encoder $z_t = g_{\textrm{enc}}(x_t) $ are summarized with an autoregressive module into context vectors $c_t = g_{\textrm{ar}}(z_{\leq t})$.
For all $t$ and all prediction steps $k$, we sample a set $X$ of size $N$ from the training data that contains one positive pair $(x_t,x_{t+k})$ and $N-1$ negative pairs $(x_t,x_j)$, where $x_j$ is randomly sampled from the same mini batch. The \gls{cpc} loss contrasts linear $k$-step future predictions $W_kc_t$ against negative samples:
\begin{equation}
    \label{eq:cpc_loss}
    \mathcal{L}_{\text{CPC}} = -\mathbb{E}_{X \sim \mathcal{D}} \Bigg[ \log \frac{\exp(z_{t+k}^{T}W_kc_t)}{\sum_{X} \exp(z_{j}^{T}W_kc_t)} \Bigg].
\end{equation}
It encourages the context representation $c_t$ to be predictive of nearby local representations $z_{t+k}$. 
Optimizing \Cref{eq:cpc_loss} relates to maximizing the mutual information \citep{tschannen2019mutual} between the context representation $c_t$ and nearby time points $x_{t+k}$ to produce good representations ($z_t$ and $c_t$) that can be used in downstream tasks, including \gls{ad}.

\subsection{Local Neural Transformations}
\label{sec:transformation_learning}
\begin{figure}[t]
    \centering
    \includegraphics[width=\linewidth]{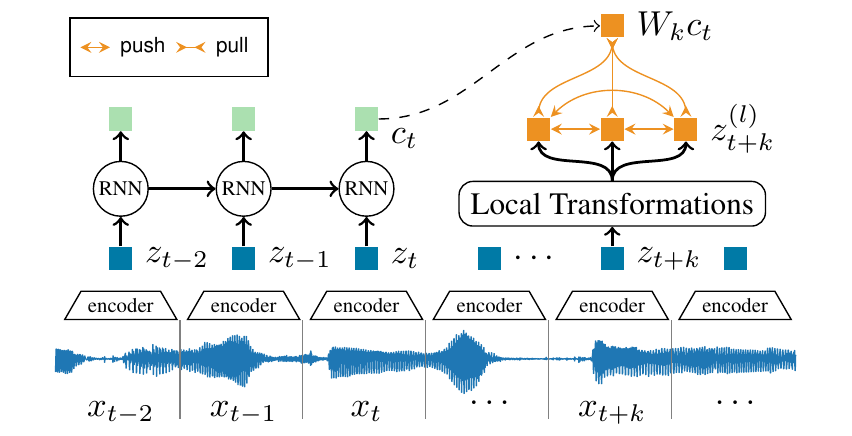}
    \caption{\gls{lnt} on latent representations $z_t$ resulting in transformed views $\mathcal{T}_l(z_t)$ - it can be viewed as pushing and pulling representations in latent space with the Dynamic Deterministic Contrastive Loss (DDCL)}
    \label{fig:ddcl_push_pull}
\end{figure}

The second part of the \gls{lnt} architecture (as shown in \Cref{fig:method_overview}) introduces an auxiliary task for \gls{ad}. The time series representations $z_t$ are processed by local neural transformations to produce different views of each embedding. This operation relates to data augmentation but has two major differences: First, the transformations are not applied at the data level but in the latent space, producing {\em latent views} of each time window. Second, the transformations are not hand-crafted as is often done in computer vision, where rotation, cropping and blurring are popular augmentations, but are instead directly learned during training \citep{tamkin2020viewmaker, qiu2021neural}.

The neural transformations are $L$ neural networks $\mathcal{T}_l(\cdot)$ with parameters $\theta_l$. They are applied to each latent representation $z_t$ to produce different latent views $z_t^{(l)} = \mathcal{T}_l(z_t)$, as shown in \Cref{fig:ddcl_push_pull}.
Each of the transformed views is encouraged to be predictive of the context at different time horizons $k$ by a loss contribution
\begin{equation*}
    \label{eq:partial_ddcl_loss}
    \ell_t^{(k,l)}(x_{\leq t}) = - \log \frac{{h\big(z_{t}^{(l)}, W_k c_{t - k} \big)}}{h\big(z_{t}^{(l)}, W_k c_{t - k} \big) + {\underset{m \neq l}{\sum} h \big( z_{t}^{(l)}, z_{t}^{(m)} \big)}},
\end{equation*}
which simultaneously pushes different views of the same latent representations apart from each other.
The notation $h(z_i, z_j) \coloneqq \exp{\frac{z_i^T z_j}{\|z_i\| \|z_j\|}}$ 
 is defined as the exponentiated cosine similarity in the embedding space.
Unlike most contrastive losses, where the negative samples are drawn from a noise distribution \citep{gutmann2012noise}, the other views to contrast against are constructed deterministically from the same input \citep{qiu2021neural}. 
The loss contributions of each time-step $t$, each transformation $l$, and each time horizon $k$ are combined to produce the \textit{Dynamic Deterministic Contrastive Loss} (\gls{ddcl}): 
\begin{equation}
    \label{eq:ddcl_loss}
    \mathcal{L}_{\text{DDCL}} = \mathbb{E}_{x_{1:T} \sim \mathcal{D}} \Bigg[ \sum_{k=1}^K \sum_{t=1}^{T} \sum_{l =1}^{L}  \ell_t^{(k,l)}(x_{\leq t}) \Bigg].
\end{equation}
During training, the two objectives (\Cref{eq:cpc_loss,eq:ddcl_loss}) are optimized jointly using a unified loss, 
\begin{equation}
    \label{eq:unified_loss}
    \mathcal{L} = \mathcal{L}_{\text{CPC}} + \lambda \cdot \mathcal{L}_{\text{DDCL}}
\end{equation}
and a balancing hyperparameter $\lambda$. 

As depicted by orange arrows in \Cref{fig:ddcl_push_pull}, $\mathcal{L}_{\text{DDCL}}$ can intuitively be interpreted as pushing and pulling different representations in latent space. The numerator pulls the learned transformations $z_{t+k}^{(l)}$ close to $W_k c_t$ ensuring semantic views, while the denominator pushes different views apart, ensuring diversity in the learned transformations.

\subsubsection{Scoring of Anomalies}
\label{sec:anomaly_scoring}
After training \gls{lnt} on a dataset of typical time series, we can use the \gls{ddcl} for \gls{ad}.
Given a test sequence $x_{1:T}$, we evaluate the contribution of individual time steps to $\mathcal{L}_{\text{DDCL}}$ (\Cref{eq:ddcl_loss}). The score for each time point $t$ in the sequence is, 
\begin{align}
\ell_t (x_{\leq t})  = \sum_{k=1}^K \sum_{l=1}^L \ell_{t}^{(k,l)}(x_{\leq t})  
\end{align}
The higher the score, the more likely the series exhibits abnormal behavior at time $t$. 
Unlike \gls{cpc}-based \gls{ad} \citep{de2021contrastive}, this anomaly score has the advantage of being \emph{deterministic} and thus there is no need to draw negative samples from a proposal or noise distribution.

\section{Analysis} 
\label{sec:theory}

Our experiments in \Cref{sec:ablation_study} show that \gls{lnt} empirically outperforms \gls{cpc} on various \gls{ad} tasks. However, since the \gls{lnt} architecture (\Cref{fig:method_overview}) is trained on two losses jointly (the DDCL and CPC losses), the natural question arises: are \emph{both} losses necessary or could we just train on the DDCL loss alone? The following analysis demonstrates the value of considering both losses jointly.


The following theorem shows that, if we trained the \gls{lnt} architecture (i.e. the encoder and transformations $\mathcal{T}_i$) only on the $\mathcal{L}_\text{DDCL}$ loss (without the $\mathcal{L}_\text{CPC}$ loss), the optimal solution would collapse to a constant encoder, a phenomenon known as the \emph{manifold collapse} in deep \gls{ad} \citep{ruff2018deep}. Thus the CPC loss acts as a regularizer in our DDCL framework to avoid the manifold collapse; it is thus strictly necessary.

\begin{theorem}
Let $g_{enc}^\theta$ and $g_{ar}^{\theta}$ be arbitrary encoders (including biases) with learned parameters $\theta$, and let $\mathcal{L}_\text{DDCL}^\theta$ be the corresponding DDCL loss. Then there exist constant encoders $g_{enc}^{\tilde{\theta}}$ and $g_{ar}^{\tilde{\theta}}$ (i.e., $\exists\tilde{\theta},a,b~\forall x,z:g_{enc}^{\tilde{\theta}}(x)=a,g_{ar}^{\tilde{\theta}}(z)=b$) with
$$ \mathcal{L}_\text{DDCL}^{\tilde{\theta}} \leq \mathcal{L}_\text{DDCL}^\theta.$$
\label{c1}
\end{theorem}

\begin{proof}
Let $g_{enc}^\theta$ and $g_{ar}^{\theta}$ be arbitrary encoders (including biases) with learned parameters $\theta$ (for notational simplicity of the proof we understand the additional parameter $W$ as included into $\theta$), and let $\mathcal{L}_\text{DDCL}^\theta$ be the corresponding DDCL loss. We observe from \Cref{eq:ddcl_loss} that $\mathcal{L}_\text{DDCL}^\theta$ decomposes into a sum of loss contributions $\ell_t^{(k, l)}(x_{\leq t};\theta)$. Let
\begin{align}\label{eq:opt}
    (x_{\leq t^*}^*, k^*, t^*) = \argmin \sum_{l=1}^L \ell_t^{(k, l)}(x_{\leq t};\theta),
\end{align}
be the indices of the summands with the smallest contribution to the sum, for a given fixed $\theta$. This means $x^{*}$ is the sample, $k^*$ the time horizon, and $t^*$ the time point associated with the smallest loss contribution to $\mathcal{L}_\text{DDCL}$. Put 
\begin{align}\label{eq:defell}
   \ell^*:=\sum_{l=1}^L \ell_t^{(k^*,l)}(x_{\leq t^*};\theta).
\end{align}
Since our encoders are equipped with bias terms there exist constant encoders $g_{enc}^{\tilde{\theta}}$ and $g_{ar}^{\tilde{\theta}}$ (i.e., $\exists\tilde{\theta},a,b\forall x,z:g_{enc}^{\tilde{\theta}}(x)=a,g_{ar}^{\tilde{\theta}}(z)=b$) with 
\begin{align}\label{eq:concell}
  \forall x,k,t: \sum_{l=1}^L \ell_t^{(k, l)}(x_{\leq t};\tilde{\theta}) = \ell^*.
\end{align}
Then we have:
\begin{align*}
    \mathcal{L}_{\text{DDCL}}(\theta) &\stackrel{\eqref{eq:ddcl_loss}}{=}  \mathbb{E} \Bigg[ \sum_{k=1}^K \sum_t^{T} \sum_{l =1}^{L}  \ell_t^{(k,l)}(x_{\leq t};\theta) \Bigg] \\ 
    &\stackrel{\eqref{eq:opt}}{\geq}  K T \sum_{l=1}^L \ell_t^{(k^*, l)}(x^*_{\leq t^*};\theta) ~\stackrel{\eqref{eq:defell}}{=}~ K T \ell^*  \\ 
    &~\stackrel{\eqref{eq:concell}}{=}~  \mathbb{E} \Bigg[ \sum_{k=1}^K \sum_t^{T} \sum_{l =1}^{L}  \ell_t^{(k,l)}(x_{\leq t};\tilde{\theta}) \Bigg]
    \stackrel{\eqref{eq:ddcl_loss}}{=} \mathcal{L}_{\text{DDCL}}(\tilde{\theta}),
\end{align*}
which was to prove. 
\end{proof}

The above theorem shows that if \gls{lnt} was trained on the \gls{ddcl} loss only, \gls{lnt} would collapse into a trivial solution. 
On the other hand a constant encoder clearly does not optimize the maximum mutual information criterion \citep{oord2018representation}, which is induced by the \gls{cpc} objective.

Besides this hard mathematical evidence, there are also other good reasons to include the CPC loss into \gls{lnt}. For instance, it ensures that the latent representations account for dynamics at longer time scales. This task is carried out by CPC's autoregressive module. Our hypothesis is that, for effective AD within time series, it is necessary to consider both: the local signal in a time window and the larger context across time windows. Otherwise the observations within a time window could be perfectly normal while not making sense in the context of a longer time horizon. For this reason, we believe that there are two types of semantic requirements of the representations and the latent views of \gls{lnt}:
\begin{itemize}
    \item \emph{Local semantics}: views should share relevant semantic information with the current time window.  (Addressed by $\mathcal{L}_\text{CPC}$)
    \item \emph{Contextualized semantics}: views should reflect how the time window relates to the rest of the time series at different, longer time horizons. (Addressed by $\mathcal{L}_\text{DDCL}$)
\end{itemize}
Both loss contributions of \gls{lnt} facilitate these requirements. \Gls{cpc} contributes local latent representations and context representations.
The semantic content of the views is managed by the \gls{ddcl} loss. Especially its numerator ensures contextualized semantics:
the views $z_t^{(l)}$ should be close to different context information $W_k c_{t-k}$ with various lags $k$.
This gives the \gls{lnt} architecture a lever to consider longer time horizons from different (non-transformed) contexts $c_{t-k}$ when deciding whether there is an anomaly at time $t$ meaning that it too exhibits \textit{contextual semantic}.

\section{Experiments}
\label{sec:experiments}
For experimental evaluation of \gls{lnt} in comparison to other methods, we study three challenging datasets. We first describe the datasets, baselines and implementation details. In \Cref{sec:results}, we present our findings: \gls{lnt} outperforms many strong baselines in detecting anomalies in the operation of a water distribution and a water treatment system and accurately finds anomalies in speech. In \Cref{sec:transformation_visualization}, we provide visualizations of the local transformations that are learned by \gls{lnt}. Finally, in \Cref{sec:ablation_study} we analyze the performance of \gls{lnt} in comparison to \gls{cpc} based alternatives. Our findings that \gls{lnt} is consistently superior, complements our theoretical analysis in \Cref{sec:theory} on why \gls{cpc} and transformation learning should be combined.


\subsection{Datasets}
We evaluate LNT on three challenging real-world datasets, namely the \gls{wadi} \cite{ahmed2017wadi}, the \gls{swat} \citep{goh2016dataset} and the \textit{Libri Speech Collection} \citep{panayotov2015librispeech}. The first two datasets are provided with labeled anomalies in the test set. 
As recent observations in \citet{wu2020current} show, many popular datasets for time series \gls{ad} seem to be mislabeled and flawed, which results in the revival of synthetic datasets \cite{lai2021revisiting}. The Libri Speech data is augmented with \textit{realistic synthetic anomalies}.

\paragraph{Water Distribution}
The dataset is acquired from a water distribution testbed and provides a model of a scaled-down version of a large water distribution network in a city \citep{ahmed2017wadi}. 
The time series data is $112$-dimensional with readings from different sensors and actuators such as pumps and valves.
The training data consists of $14$ days of normal operation sampled with a frequency of $1$ Hz, resulting in a series length of $1048571$. The test set consists of $2$ days of additional operation ($172801$ time steps), during which $15$ attacks were staged with an average duration of $\approx 12$ minutes.

\paragraph{Secure Water Treatment}
This dataset is from a testbed for water treatment \citep{mathur2016swat} that evaluates the Cyber Security of a fully functional plant with a six-stage process of filtration and chemical dosing.
\citet{goh2016dataset} collected $11$ days of operation data.
Under normal operation $51$ sensor channels are recorded for $7$ days yielding a training time series of length $475200$. For the test data of length $224960$, $36$ attacks were launched during the last 4 days of the collection process. As suggested in \citet{goh2016dataset,li2019mad}, the first $21600$ samples from the training data are removed for training stability.

We follow the experimental setup of \citet{he2019temporal} and take the first part of the collection under attack as the validation set and drop channels which are constant in both training and test set, yielding a time series of $45$ dimension. 


\paragraph{Libri Speech}
The \emph{LibriSpeech dataset} \cite{panayotov2015librispeech} is an audio collection with spoken language recordings from $251$ distinct speakers.  
We adopt the setup of \citet{oord2018representation} with their train/test split and unsupervised training on the raw time signal without further pre-processing. 
For \gls{ad} benchmarks,  
we randomly place additive pure sine tones of varying frequency ($20$ - $120$ Hz) and length ($512$ - $4096$ time steps) in the test data, yielding consecutive anomaly regions making up $\approx 10\%$ of the test data.
Speech data offers a challenging benchmark for deep \gls{ad} methods 
 since speech typically exhibits complex temporal dynamics, 
due to high multi-modality introduced through different speakers and word sequences \citep{oord2018representation}.

\subsection{Baselines and Implementation Details}
\begin{table}[ht]
    \centering
        \small
    \begin{tabular}{|c||c|c|c|}
        \hline
        Types &  SWaT & WaDi & Libri\\
        \hline
        \# neurons & 24 & 32 & 64 \\
        \# layers & 2 & 2 & 3 \\
        activation & ReLU & ReLU & ReLU \\
        bias & False & False & False \\
        \hline
    \end{tabular}
    \caption{Neural Transformation Hyperparameters}
    \label{tab:hyperparameters}
\end{table}

\paragraph{Baselines}







We study \gls{lnt} in comparison to  different classes of \gls{ad} algorithms, ranging from classical methods to recent advances in deep \gls{ad}. They include
\begin{enumerate*}[label=(\roman*)]
    \item classical methods, such as Isolation Forests \citep{liu2008isolation}, PCA reconstruction error \citep{shyu2003novel}, and Feature Bagging \citep{lazarevic2005feature},
    \item auto-regressive future predictions with LSTM \citep{hundman2018detecting} and GDN \citep{deng2021graph}, which uses a graph to model the relations among variables as attention for the prediction,
    \item methods that estimate the density of the data, such as KNN \citep{angiulli2002fast}, LOF \citep{breunig2000lof}, combinations with deep auto-encoders DAGMM \citep{zong2018deep}, 
    \item methods that employ a one-class objective, including OC-SVM \citep{scholkopf1999support}, DeepSVDD \citep{ruff2018deep} and THOC \citep{shen2020timeseries} for time-series,
    \item methods that leverage the reconstruction of an auto-encoder with EncDec-AD \citep{malhotra2016lstm} and LSTM-VAE \citep{park2018multimodal}
    \item and finally methods that use the ability of GANs to discriminate fake examples, like BeatGAN \citep{zhou2019beatgan} and MAD-GAN \citep{li2019mad}.
\end{enumerate*}

\paragraph{Implementation Details}
For \gls{lnt}, the hyperparamaters are adopted from those reported by \citet{oord2018representation} for CPC: especially $c_t \in \mathbb{R}^{256}$, $z_t \in \mathbb{R}^{512}$ and $K=12$ for experiments with LibriSpeech data. The data is processed in sub-sequences of length $20480$ for both training and testing.
Since the other datasets contain way less diverse data points and show simpler temporal dynamics, the embeddings size, and thus the capacity of the model, is reduced to $c_t \in \mathbb{R}^{32}$, $z_t \in \mathbb{R}^{128}$. Also, the time-convolutional encoder network is down-sized to filters $(3,3,4,2)$ and strides $(3,3,4,2)$ resulting in the convolution of $72$ time steps.

We consistently choose $L=12$ distinct learned transformations $T_l(z_t)$ for all datasets. Each is represented by an \emph{MLP} with properties summarized in table \ref{tab:hyperparameters}. The final layer always shares the dimensionality of $z_t$ and is applied as a \emph{multiplicative mask} with \emph{sigmoid} activation to it.
Additional implementation details are in the appendix.

\begin{table*}[t!]
    \centering
    \small
    \begin{tabular}{c|c|c|c|c|c|c|c|c|c|c|c}
         & \small{LOF} & \small{OCSVM} & \small{IF} & \small{DeepSVDD} & \small{DAGMM} & \small{EncDec} & \small{VAE} & \small{MAD-GAN} & \small{BeatGAN} & \small{THOC} & \small{LNT (ours)}  \\
         \hline
         \hline
         $F_1$ & 86.36 & 75.98 & 85.00 & 82.82 & 85.38 & 75.56 & 86.39 & 86.89 & 81.95 & 88.09 & \textbf{88.65} \\
    \end{tabular}
    \caption{F1-scores ($\%$) for the Secure Water Treatment Dataset (SWaT). Baseline results as reported in \citet{shen2020timeseries}.}
    \label{tab:results:swat}
    \vspace{-5pt}
\end{table*}

\subsection{Results}
\label{sec:results}
We judge the anomaly scores predicted by the algorithms for each time step individually.
Since the ratio of anomalies is imbalanced in the data, we evaluated the prediction performance with the $F_1$ score, consistent with previous work.
Additionally, we also report results using the ROC curve. The area under the curve (ROC-AUC) is a metric to judge the quality of the anomaly score independent of the choice of threshold, which is specifically chosen for its additional insights beyond the evaluation of a single threshold.

\begin{table}[t!]
\small
    \centering
        \begin{tabular}{ |c||c|c|c|c| }
         \hline
          \textbf{Method} & $F_1$ & Prec & Rec \\
         \hline
         \hline
            PCA &  0.10 & 39.53 & 5.63 \\
            KNN &  0.08 & 7.76 & 7.75 \\
            FB &  0.09 & 8.60 & 8.60 \\
            EncDec-AD &  0.34 & 34.35 & 34.35 \\
            DAGMM &  0.36 & 54.44 & 26.99 \\
            LSMT-VAE &  0.25 & 87.79 & 14.45 \\
            MAD-GAN & 0.37 & 41.44 & 33.92 \\
            GDN & \textbf{0.57} & \textbf{97.50} & 40.19 \\
        \hline
            LNT (ours) & 0.39 & 29.34 & \textbf{60.92} \\  
         \hline
        \end{tabular}
        \caption{Experimental Results on the Water Distribution Data (WaDi). Baseline results from \citet{deng2021graph}.}
        \label{tab:results:wadi_final}
        \vspace{-5pt}
\end{table}

\begin{table}[t!]
\small
    \centering
        \begin{tabular}{ |c||c|c|c|c|}
         \hline
          \textbf{Method} & AUC & Prec & Rec & $F_1$  \\
         \hline
         \hline
            LSTM & 0.58 & 15.0 & 15.0 & 0.15\\
            THOC & 0.82 & 30.2 & 30.0 & 0.30\\
         \hline
            LNT (ours) & \textbf{0.93} & \textbf{65.0} &\textbf{65.0} & \textbf{0.65}  \\  
         \hline
        \end{tabular}
        \caption{Experimental Results on synthetic anomalies randomly placed in the LibriSpeech dataset.}
        \label{tab:results:speech}
        \vspace{-5pt}
\end{table}

The results on the \gls{swat} and \gls{wadi} datasets can be seen in \Cref{tab:results:swat,tab:results:wadi_final}, respectively. The ROC curves of our method on the \gls{swat} and \gls{wadi} datasets are provided in \Cref{fig:cpc_scoring_swat,fig:cpc_scoring_wadi}. For \gls{swat},
our approach (LNT) outperformed a set of challenging baselines as reported by \citet{shen2020timeseries} with the highest $F_1$ score ($88.65\%$).
Meanwhile for \gls{wadi}, our model produces comparable results both in terms of $F_1$ and precision, with the highest recall value\footnote{In all experiments (and methods) the thresholds on the continuous anomaly score are optimized for the best $F1$.}.
Notably, GDN achieves the highest precision on \gls{wadi} \footnote{Results reported for GDN seem to be hard to reproduce (see \url{https://github.com/d-ailin/GDN/issues/9}).}, but has a lower recall than our method. In many mission-critical applications, detecting as many anomalies as possible is often much more important, as a false negative can do more harm than a false positive. This makes the high recall of \gls{lnt} ($60.92\%$) preferable, while retaining an acceptably high $F1$ score.

We argue that the novel criterion for AD based on contrasting learned latent data transformations allows \gls{lnt} to also uncover some of the harder detectable anomalies in the dataset.
A similar behaviour can also be observed for the LibriSpeech data with results in terms of ROC curves shown in \Cref{fig:roc_results}. Here, \gls{lnt} clearly outperforms both deep learning methods. This shows that detecting anomalies within speech data with its complex temporal dynamics is indeed a challenging task for many deep AD algorithms. Especially the future predictions of LSTM perform only slightly better than random chance in this experiment for all possible thresholds. This emphasizes the benefit of contrasting of neural transformations to uncover such hard anomalies. Additional metrics for this experiment are reported in \Cref{tab:results:speech}. 

\begin{figure}[t!]
    \centering
    \includegraphics[width=0.75\linewidth]{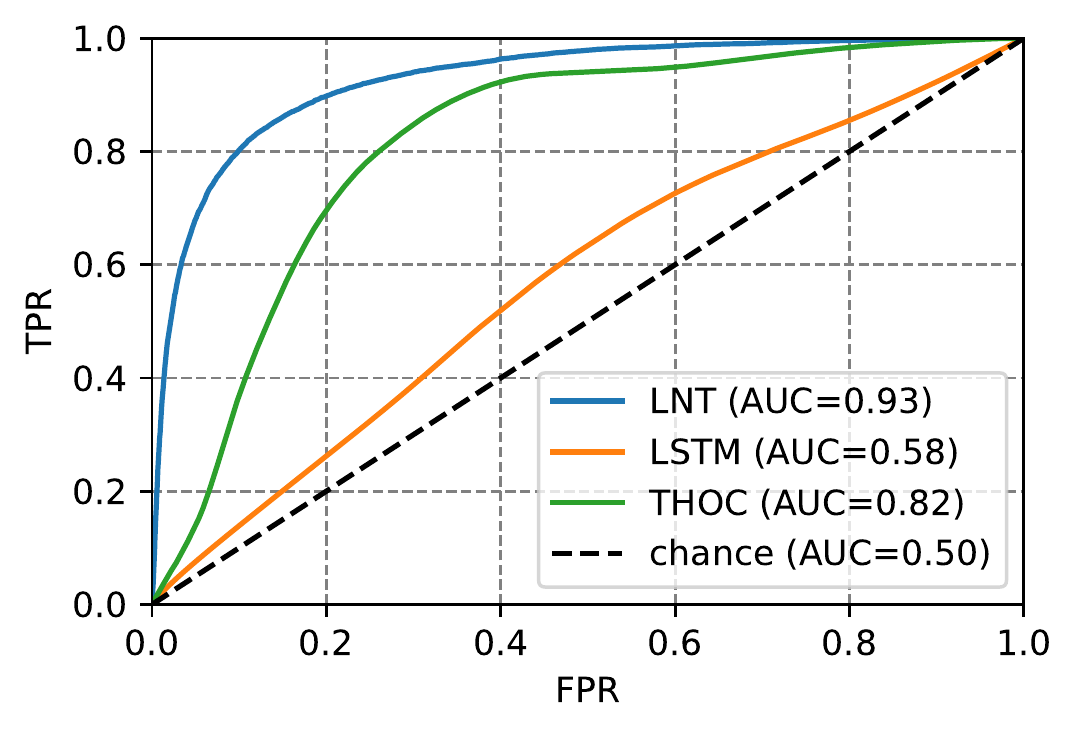}
    \caption{Our approach \gls{lnt} outperforms deep baselines in \gls{ad} on speech data in terms of ROC-AUC curves.}
    \label{fig:roc_results}
    \vspace{-10pt}
\end{figure}

\newcommand{\mcrot}[4]{\multicolumn{#1}{#2}{\rlap{\rotatebox{#3}{#4}~}}}

\begin{figure*}[t!]
    \centering
    \includegraphics[width=0.9\textwidth]{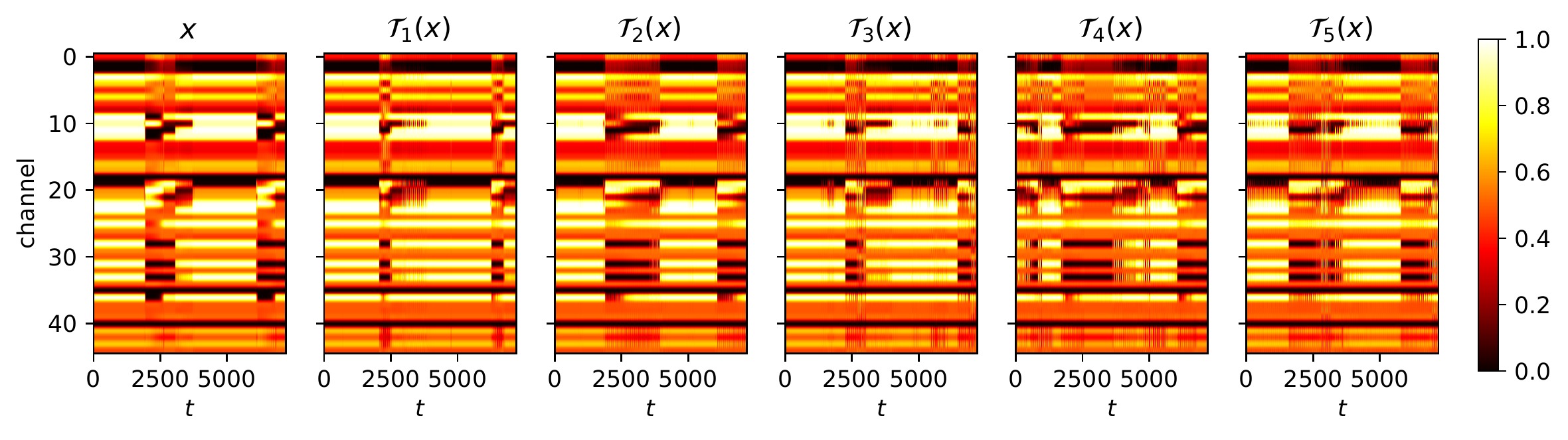}
    \caption{Visualizations of selected transformations in data-space that show semantically interpretable behaviour, such as altered delays in specific channels. Representations from \gls{swat} dataset are decoded with a seperatly trained auto-encoder.}
    \label{fig:transformation_visualization}
    \vspace{-5pt}
\end{figure*}

\subsection{Visualization of Transformations}
\label{sec:transformation_visualization}
\begin{figure*}[ht]
     \centering
     \begin{subfigure}[b]{0.28\textwidth}
         \centering
         \includegraphics[width=\linewidth]{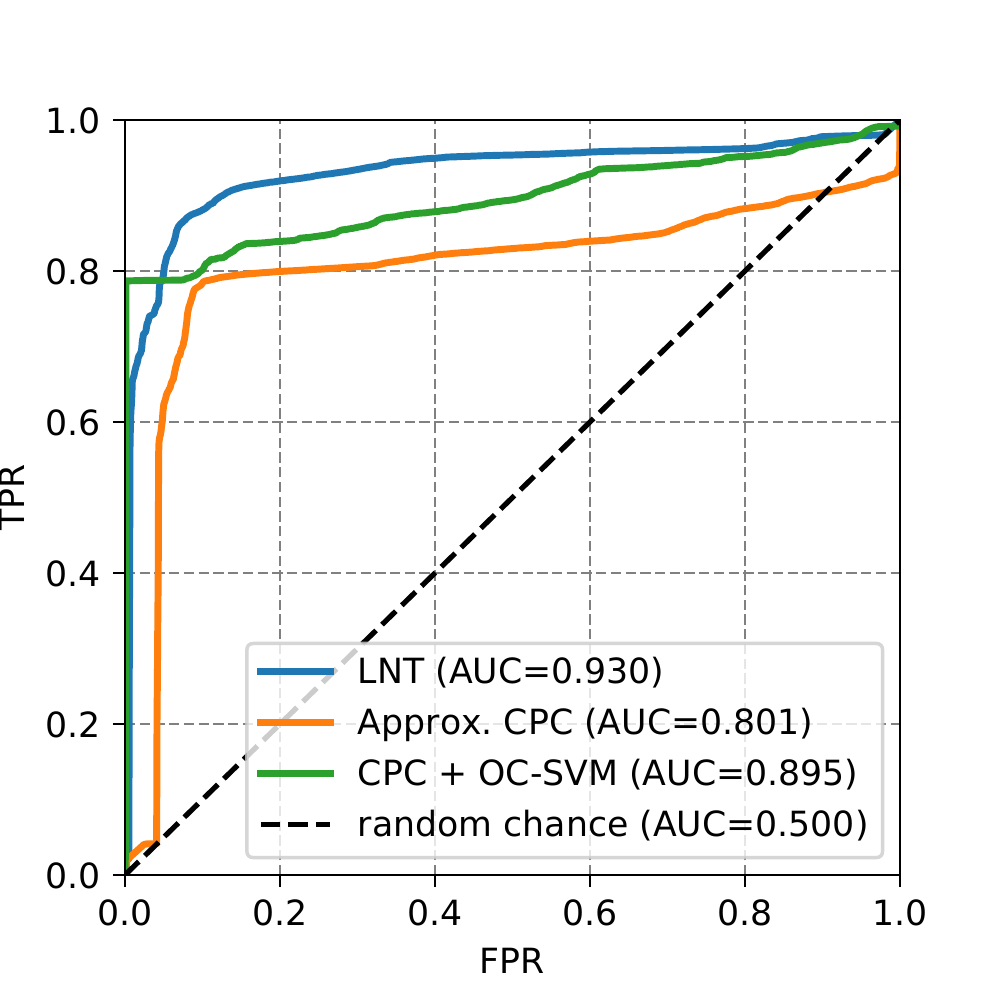}
         \caption{SWaT}
         \label{fig:cpc_scoring_swat}
     \end{subfigure}
     \begin{subfigure}[b]{0.28\textwidth}
         \centering
         \includegraphics[width=\linewidth]{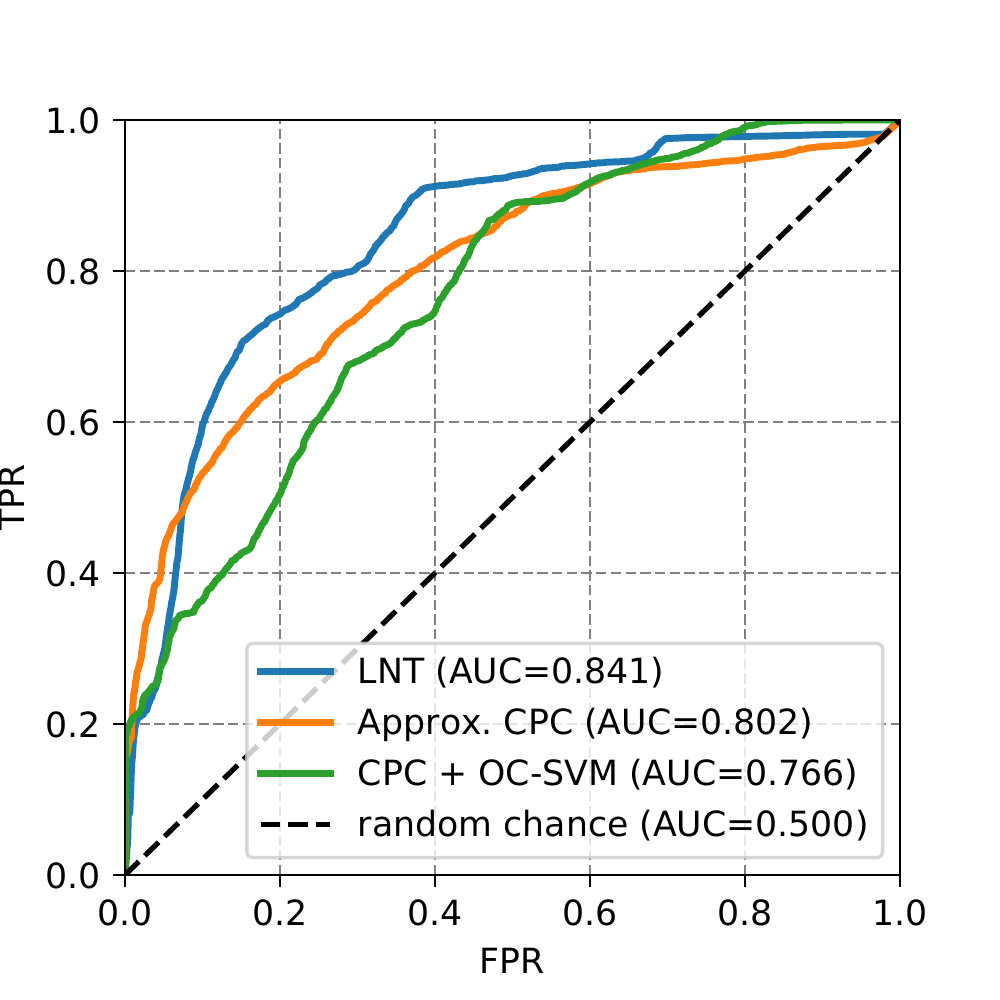}
         \caption{WaDi}
         \label{fig:cpc_scoring_wadi}
     \end{subfigure}
     \begin{subfigure}[b]{0.28\textwidth}
         \centering
         \includegraphics[width=\linewidth]{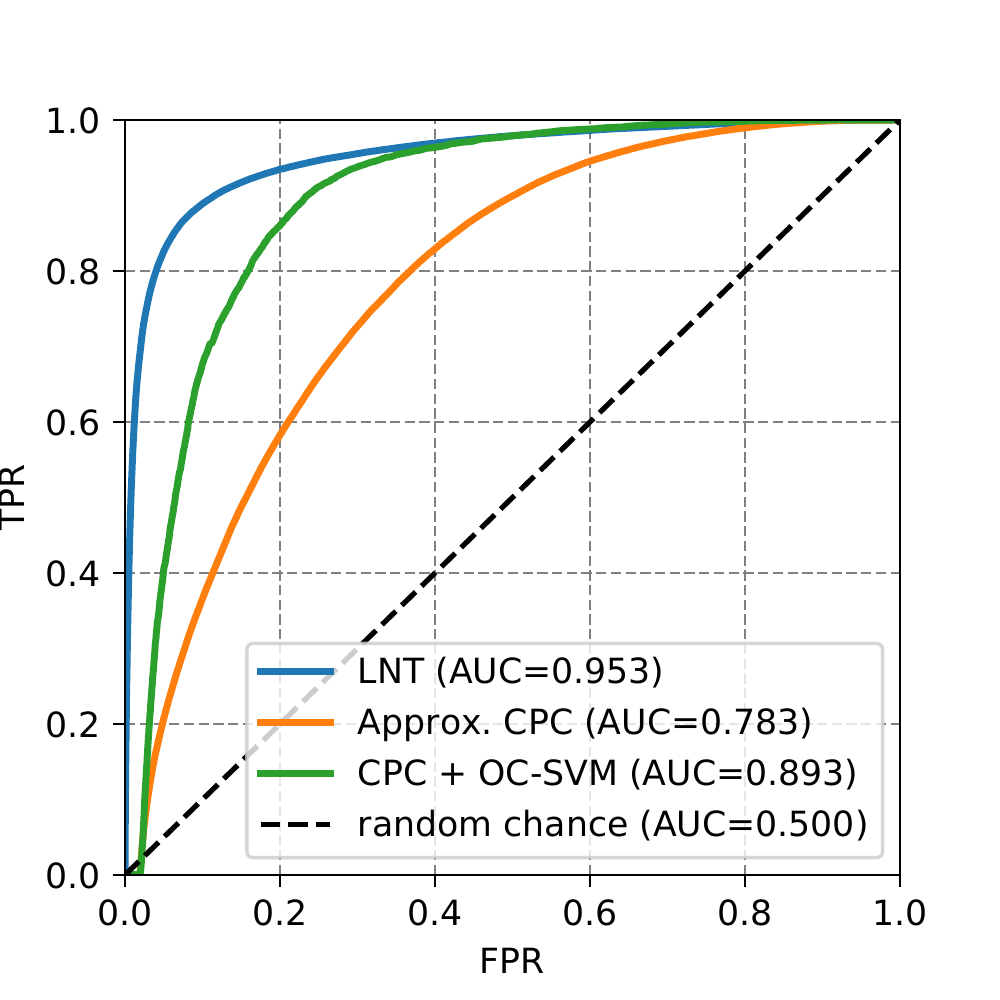}
         \caption{LibriSpeech}
         \label{fig:cpc_scoring_libri}
     \end{subfigure}
    \caption{Improvement of LNT over CPC scoring evaluated for different datasets. The combination of transformation learning with local representation learning of CPC consitently outperforms the other variants of CPC for anomaly scoring.}
    \label{fig:cpc_scoring}
    \vspace{-5pt}
\end{figure*}
In general, it is considered hard to get insights from embedding visualizations for $z_t$ in the latent space. Hence, to make the transformations interpretable in terms of semantics, we propose to visualize them in data space. We reuse the encoder as described in  \Cref{sec:representation_learning} and enrich it with a separate decoder. We train the decoder to reconstruct the (non-transformed) input data while freezing the encoder weights. The trained decoder is then applied to transformed embeddings to visualize them in data space.

We chose a subset $\{ \mathcal{T}_i \}_{i=1}^5$ of five transformations which showed interpretable behavior in experiments with \textit{\gls{swat}} as shown in \Cref{fig:transformation_visualization}: For the non-transformed series $x$ the signal jumps in channels $25$ and $36$ at $t \approx 2500$. This jump is delayed for channels $26-35$. Interestingly, we found that this delay is altered by the learned transformations. For example, $\mathcal{T}_1$ removes this delay causing the signal jump for all aforementioned channels at $t \approx 2500$. In contrast, $\mathcal{T}_2$ affects the series oppositely by enlarging this delay. 

In summary, these transformations produce \textit{semantically} meaningful and \textit{diverse} views of the time series.
Admittedly, current interpretations are still rather high-level and fairly limited from application standpoints. However, without domain knowledge, there exists no \textit{gold standard} for a good transformation on the data to compare against. This was the original motivation for the usage of \textit{learnable transformations}, as effective data augmentation for \gls{ad}.

\subsection{Comparison of \gls{cpc} based \gls{ad} Methods}
\label{sec:ablation_study}
Finally, we study the advantage of \gls{lnt} over \gls{cpc}. There are various ways to use \gls{cpc} for \gls{ad}. Beyond \gls{lnt}, we consider two methods that build on \gls{cpc}: 
\begin{enumerate*}[label=(\roman*)]
    \item methods that directly use the \gls{cpc}-loss to score anomalies \citep{de2021contrastive} and
    \item methods that use \gls{cpc} as a feature extractor and then run another \gls{ad} method such as \acrshort{ocsvm} on the extracted features. 
\end{enumerate*}
One disadvantage of (i) are the negative samples. They make it nontrivial to evaluate the \gls{cpc}-loss on test data. We employ a practical implementation (Approx. CPC) without negative samples at test time. \citet{de2021contrastive} argue that taking samples from the test data is biased and using the training data is infeasible in practice.



In contrast, \gls{ddcl} is deterministic and the alternative views are all constructed from a single sample. It is hence straightforward to use it to score anomalies at test time.
From the results in \Cref{fig:cpc_scoring}, we found that the combination of transformation learning with local representation learning of \gls{cpc} consistently outperforms the considered variants of \gls{cpc} for \gls{ad} in all three datasets.
This connects to the discussion about \textit{contextualized semantics} in \Cref{sec:theory}. Comparing \gls{lnt} with CPC + OC-SVM supports our claim:  While the OC-SVM with \gls{cpc} input features has access only to  the \textit{local semantics} in the \gls{cpc} representations, the performance of \gls{lnt} in \Cref{fig:cpc_scoring} is consistently superior and can be explained by its transformations exhibiting both \textit{contextualized semantics} and \textit{diversity}.

\section{Conclusion}
We propose a novel self-supervised method, \gls{lnt}, to detect anomalies within time series. The key ingredient is a novel training objective combining representation and transformation learning. 
We prove that both  learning paradigms complement each other to avoid trivial solutions not appropriate for \gls{ad}. We find in an empirical study that \gls{lnt} learns to insert delays, which allows it to outperform many strong baselines on challenging detection tasks.

\section*{Acknowledgements}
Marius Kloft acknowledges support by the Carl-Zeiss Foundation, the DFG awards KL 2698/2-1 and KL 2698/5-1, and the BMBF awards 01$|$S18051A, 03$|$B0770E, and 01$|$S21010C.

Work by Decky Aspandi Latif and Steffen Staab has been partially funded by BMBF in the research project ``XAPS - eXplainable AI for Automated Production System''.

Stephan Mandt acknowledges support by DARPA under contract No. HR001120C0021, the Department of Energy under grant DE-SC0022331, the National Science Foundation under the NSF CAREER award 2047418 and Grants 1928718, 2003237 and 2007719, as well as gifts from Intel, Disney, and Qualcomm. Any opinions, findings and conclusions or recommendations expressed in this material are those of the authors and do not necessarily reflect the views of DARPA or NSF.

The Bosch Group is carbon neutral. Administration, manufacturing and research activities do no longer leave a carbon footprint. This also includes GPU clusters on which the experiments have been performed.

\bibliography{refs}

\begin{thebibliography}{53}
\providecommand{\natexlab}[1]{#1}
\providecommand{\url}[1]{\texttt{#1}}
\expandafter\ifx\csname urlstyle\endcsname\relax
  \providecommand{\doi}[1]{doi: #1}\else
  \providecommand{\doi}{doi: \begingroup \urlstyle{rm}\Url}\fi

\bibitem[Ahmed et~al.(2017)Ahmed, Palleti, and Mathur]{ahmed2017wadi}
Ahmed, C.~M., Palleti, V.~R., and Mathur, A.~P.
\newblock Wadi: a water distribution testbed for research in the design of
  secure cyber physical systems.
\newblock In \emph{Proceedings of the 3rd International Workshop on
  Cyber-Physical Systems for Smart Water Networks}, pp.\  25--28, 2017.

\bibitem[Angiulli \& Pizzuti(2002)Angiulli and Pizzuti]{angiulli2002fast}
Angiulli, F. and Pizzuti, C.
\newblock Fast outlier detection in high dimensional spaces.
\newblock In \emph{European conference on principles of data mining and
  knowledge discovery}, pp.\  15--27. Springer, 2002.

\bibitem[Audibert et~al.(2020)Audibert, Michiardi, Guyard, Marti, and
  Zuluaga]{audibert2020usad}
Audibert, J., Michiardi, P., Guyard, F., Marti, S., and Zuluaga, M.~A.
\newblock Usad: Unsupervised anomaly detection on multivariate time series.
\newblock In \emph{Proceedings of the 26th ACM SIGKDD International Conference
  on Knowledge Discovery \& Data Mining}, pp.\  3395--3404, 2020.
\newblock \doi{10.1145/3394486.3403392}.

\bibitem[Bergman \& Hoshen(2020)Bergman and Hoshen]{bergman2020classification}
Bergman, L. and Hoshen, Y.
\newblock Classification-based anomaly detection for general data.
\newblock \emph{arXiv preprint arXiv:2005.02359}, 2020.

\bibitem[Breunig et~al.(2000)Breunig, Kriegel, Ng, and Sander]{breunig2000lof}
Breunig, M.~M., Kriegel, H.-P., Ng, R.~T., and Sander, J.
\newblock Lof: identifying density-based local outliers.
\newblock In \emph{Proceedings of the 2000 ACM SIGMOD international conference
  on Management of data}, pp.\  93--104, 2000.

\bibitem[Carmona et~al.(2021)Carmona, Aubet, Flunkert, and
  Gasthaus]{carmona2021neural}
Carmona, C.~U., Aubet, F.-X., Flunkert, V., and Gasthaus, J.
\newblock Neural contextual anomaly detection for time series.
\newblock \emph{arXiv preprint arXiv:2107.07702}, 2021.

\bibitem[de~Haan \& L{\"o}we(2021)de~Haan and L{\"o}we]{de2021contrastive}
de~Haan, P. and L{\"o}we, S.
\newblock Contrastive predictive coding for anomaly detection.
\newblock \emph{arXiv preprint arXiv:2107.07820}, 2021.

\bibitem[Deng \& Hooi(2021)Deng and Hooi]{deng2021graph}
Deng, A. and Hooi, B.
\newblock Graph neural network-based anomaly detection in multivariate time
  series.
\newblock In \emph{Proceedings of the AAAI Conference on Artificial
  Intelligence}, volume~35, pp.\  4027--4035, 2021.

\bibitem[Filonov et~al.(2016)Filonov, Lavrentyev, and
  Vorontsov]{filonov2016multivariate}
Filonov, P., Lavrentyev, A., and Vorontsov, A.
\newblock Multivariate industrial time series with cyber-attack simulation:
  Fault detection using an lstm-based predictive data model.
\newblock \emph{arXiv preprint arXiv:1612.06676}, 2016.

\bibitem[Geiger et~al.(2020)Geiger, Liu, Alnegheimish, Cuesta-Infante, and
  Veeramachaneni]{geiger2020tadgan}
Geiger, A., Liu, D., Alnegheimish, S., Cuesta-Infante, A., and Veeramachaneni,
  K.
\newblock Tadgan: Time series anomaly detection using generative adversarial
  networks.
\newblock In \emph{2020 IEEE International Conference on Big Data (Big Data)},
  pp.\  33--43. IEEE, 2020.
\newblock \doi{10.1109/BigData50022.2020.9378139}.

\bibitem[Gidaris et~al.(2018)Gidaris, Singh, and
  Komodakis]{gidaris2018unsupervised}
Gidaris, S., Singh, P., and Komodakis, N.
\newblock Unsupervised representation learning by predicting image rotations.
\newblock \emph{arXiv preprint arXiv:1803.07728}, 2018.

\bibitem[Goh et~al.(2016)Goh, Adepu, Junejo, and Mathur]{goh2016dataset}
Goh, J., Adepu, S., Junejo, K.~N., and Mathur, A.
\newblock A dataset to support research in the design of secure water treatment
  systems.
\newblock In \emph{International conference on critical information
  infrastructures security}, pp.\  88--99. Springer, 2016.

\bibitem[Golan \& El-Yaniv(2018)Golan and El-Yaniv]{golan2018deep}
Golan, I. and El-Yaniv, R.
\newblock Deep anomaly detection using geometric transformations.
\newblock In \emph{Advances in Neural Information Processing Systems}, pp.\
  9758--9769, 2018.

\bibitem[Goodfellow et~al.(2014)Goodfellow, Pouget-Abadie, Mirza, Xu,
  Warde-Farley, Ozair, Courville, and Bengio]{goodfellow2014generative}
Goodfellow, I., Pouget-Abadie, J., Mirza, M., Xu, B., Warde-Farley, D., Ozair,
  S., Courville, A., and Bengio, Y.
\newblock Generative adversarial nets.
\newblock \emph{Advances in neural information processing systems}, 27, 2014.

\bibitem[Guo et~al.(2018)Guo, Liao, Wang, Yu, Ji, and
  Li]{guo2018multidimensional}
Guo, Y., Liao, W., Wang, Q., Yu, L., Ji, T., and Li, P.
\newblock Multidimensional time series anomaly detection: A gru-based gaussian
  mixture variational autoencoder approach.
\newblock In \emph{Asian Conference on Machine Learning}, pp.\  97--112. PMLR,
  2018.
\newblock URL \url{http://proceedings.mlr.press/v95/guo18a.html}.

\bibitem[Gutmann \& Hyv{\"a}rinen(2012)Gutmann and
  Hyv{\"a}rinen]{gutmann2012noise}
Gutmann, M.~U. and Hyv{\"a}rinen, A.
\newblock Noise-contrastive estimation of unnormalized statistical models, with
  applications to natural image statistics.
\newblock \emph{Journal of Machine Learning Research}, 13\penalty0 (2), 2012.

\bibitem[He \& Zhao(2019)He and Zhao]{he2019temporal}
He, Y. and Zhao, J.
\newblock Temporal convolutional networks for anomaly detection in time series.
\newblock In \emph{Journal of Physics: Conference Series}, volume 1213, pp.\
  042050. IOP Publishing, 2019.
\newblock \doi{10.1088/1742-6596/1213/4/042050}.

\bibitem[Hendrycks et~al.(2019)Hendrycks, Mazeika, Kadavath, and
  Song]{hendrycks2019using}
Hendrycks, D., Mazeika, M., Kadavath, S., and Song, D.
\newblock Using self-supervised learning can improve model robustness and
  uncertainty.
\newblock In \emph{Advances in Neural Information Processing Systems}, pp.\
  15663--15674, 2019.

\bibitem[Hundman et~al.(2018)Hundman, Constantinou, Laporte, Colwell, and
  Soderstrom]{hundman2018detecting}
Hundman, K., Constantinou, V., Laporte, C., Colwell, I., and Soderstrom, T.
\newblock Detecting spacecraft anomalies using lstms and nonparametric dynamic
  thresholding.
\newblock In \emph{Proceedings of the 24th ACM SIGKDD international conference
  on knowledge discovery \& data mining}, pp.\  387--395, 2018.

\bibitem[Kingma \& Welling(2014)Kingma and Welling]{kingma2014auto}
Kingma, D.~P. and Welling, M.
\newblock Auto-encoding variational bayes.
\newblock In Bengio, Y. and LeCun, Y. (eds.), \emph{2nd International
  Conference on Learning Representations, {ICLR} 2014, Conference Track
  Proceedings}, Banff, AB, Canada, 2014.
\newblock URL \url{http://arxiv.org/abs/1312.6114}.

\bibitem[Lai et~al.(2021)Lai, Zha, Xu, Zhao, Wang, and Hu]{lai2021revisiting}
Lai, K.-H., Zha, D., Xu, J., Zhao, Y., Wang, G., and Hu, X.
\newblock Revisiting time series outlier detection: Definitions and benchmarks.
\newblock 2021.

\bibitem[Lazarevic \& Kumar(2005)Lazarevic and Kumar]{lazarevic2005feature}
Lazarevic, A. and Kumar, V.
\newblock Feature bagging for outlier detection.
\newblock In \emph{Proceedings of the eleventh ACM SIGKDD international
  conference on Knowledge discovery in data mining}, pp.\  157--166, 2005.

\bibitem[Li et~al.(2019)Li, Chen, Jin, Shi, Goh, and Ng]{li2019mad}
Li, D., Chen, D., Jin, B., Shi, L., Goh, J., and Ng, S.-K.
\newblock Mad-gan: Multivariate anomaly detection for time series data with
  generative adversarial networks.
\newblock In \emph{International Conference on Artificial Neural Networks},
  pp.\  703--716. Springer, 2019.
\newblock \doi{10.1016/j.neucom.2020.10.084}.
\newblock URL
  \url{https://www.sciencedirect.com/science/article/pii/S0925231220316970}.

\bibitem[Liang et~al.(2021)Liang, Song, Wang, Guo, Li, and
  Liang]{liang2021robust}
Liang, H., Song, L., Wang, J., Guo, L., Li, X., and Liang, J.
\newblock Robust unsupervised anomaly detection via multi-time scale dcgans
  with forgetting mechanism for industrial multivariate time series.
\newblock \emph{Neurocomputing}, 423:\penalty0 444--462, 2021.
\newblock \doi{10.1016/j.neucom.2020.10.084}.

\bibitem[Liu et~al.(2008)Liu, Ting, and Zhou]{liu2008isolation}
Liu, F.~T., Ting, K.~M., and Zhou, Z.-H.
\newblock Isolation forest.
\newblock In \emph{2008 eighth ieee international conference on data mining},
  pp.\  413--422. IEEE, 2008.

\bibitem[Malhotra et~al.(2015)Malhotra, Vig, Shroff, and
  Agarwal]{malhotra2015long}
Malhotra, P., Vig, L., Shroff, G., and Agarwal, P.
\newblock Long short term memory networks for anomaly detection in time series.
\newblock In \emph{Proceedings}, volume~89, pp.\  89--94, 2015.

\bibitem[Malhotra et~al.(2016)Malhotra, Ramakrishnan, Anand, Vig, Agarwal, and
  Shroff]{malhotra2016lstm}
Malhotra, P., Ramakrishnan, A., Anand, G., Vig, L., Agarwal, P., and Shroff, G.
\newblock Lstm-based encoder-decoder for multi-sensor anomaly detection.
\newblock \emph{arXiv preprint arXiv:1607.00148}, 2016.

\bibitem[Mathur \& Tippenhauer(2016)Mathur and Tippenhauer]{mathur2016swat}
Mathur, A.~P. and Tippenhauer, N.~O.
\newblock Swat: A water treatment testbed for research and training on ics
  security.
\newblock In \emph{2016 international workshop on cyber-physical systems for
  smart water networks (CySWater)}, pp.\  31--36. IEEE, 2016.

\bibitem[Munir et~al.(2019)Munir, Siddiqui, Dengel, and
  Ahmed]{munir2019deepant}
Munir, M., Siddiqui, S.~A., Dengel, A., and Ahmed, S.
\newblock Deepant: A deep learning approach for unsupervised anomaly detection
  in time series.
\newblock \emph{IEEE Access}, 7:\penalty0 1991--2005, 2019.
\newblock \doi{10.1109/ACCESS.2018.2886457}.

\bibitem[Nalisnick et~al.(2018)Nalisnick, Matsukawa, Teh, Gorur, and
  Lakshminarayanan]{nalisnick2018deep}
Nalisnick, E., Matsukawa, A., Teh, Y.~W., Gorur, D., and Lakshminarayanan, B.
\newblock Do deep generative models know what they don't know?
\newblock In \emph{International Conference on Learning Representations}, 2018.

\bibitem[Niu et~al.(2020)Niu, Yu, and Wu]{niu2020lstm}
Niu, Z., Yu, K., and Wu, X.
\newblock Lstm-based vae-gan for time-series anomaly detection.
\newblock \emph{Sensors}, 20\penalty0 (13):\penalty0 3738, 2020.
\newblock URL \url{https://www.mdpi.com/1424-8220/20/13/3738}.

\bibitem[Oord et~al.(2018)Oord, Li, and Vinyals]{oord2018representation}
Oord, A. v.~d., Li, Y., and Vinyals, O.
\newblock Representation learning with contrastive predictive coding.
\newblock \emph{arXiv preprint arXiv:1807.03748}, 2018.

\bibitem[Panayotov et~al.(2015)Panayotov, Chen, Povey, and
  Khudanpur]{panayotov2015librispeech}
Panayotov, V., Chen, G., Povey, D., and Khudanpur, S.
\newblock Librispeech: an asr corpus based on public domain audio books.
\newblock In \emph{2015 IEEE international conference on acoustics, speech and
  signal processing (ICASSP)}, pp.\  5206--5210. IEEE, 2015.

\bibitem[Park et~al.(2018)Park, Hoshi, and Kemp]{park2018multimodal}
Park, D., Hoshi, Y., and Kemp, C.~C.
\newblock A multimodal anomaly detector for robot-assisted feeding using an
  lstm-based variational autoencoder.
\newblock \emph{IEEE Robotics and Automation Letters}, 3\penalty0 (3):\penalty0
  1544--1551, 2018.

\bibitem[Pereira \& Silveira(2018)Pereira and
  Silveira]{pereira2018unsupervised}
Pereira, J. and Silveira, M.
\newblock Unsupervised anomaly detection in energy time series data using
  variational recurrent autoencoders with attention.
\newblock In \emph{2018 17th IEEE international conference on machine learning
  and applications (ICMLA)}, pp.\  1275--1282. IEEE, 2018.
\newblock \doi{10.1109/ICMLA.2018.00207}.

\bibitem[Qiu et~al.(2021)Qiu, Pfrommer, Kloft, Mandt, and
  Rudolph]{qiu2021neural}
Qiu, C., Pfrommer, T., Kloft, M., Mandt, S., and Rudolph, M.
\newblock Neural transformation learning for deep anomaly detection beyond
  images.
\newblock In \emph{International Conference on Machine Learning}, pp.\
  8703--8714. PMLR, 2021.

\bibitem[Rudolph et~al.(2017)Rudolph, Ruiz, Athey, and
  Blei]{rudolph2017structured}
Rudolph, M., Ruiz, F., Athey, S., and Blei, D.
\newblock Structured embedding models for grouped data.
\newblock In \emph{Neural Information Processing Systems}, 2017.

\bibitem[Ruff et~al.(2018)Ruff, Vandermeulen, Goernitz, Deecke, Siddiqui,
  Binder, M{\"u}ller, and Kloft]{ruff2018deep}
Ruff, L., Vandermeulen, R., Goernitz, N., Deecke, L., Siddiqui, S.~A., Binder,
  A., M{\"u}ller, E., and Kloft, M.
\newblock Deep one-class classification.
\newblock In \emph{International conference on machine learning}, pp.\
  4393--4402, 2018.

\bibitem[Ruff et~al.(2020)Ruff, Kauffmann, Vandermeulen, Montavon, Samek,
  Kloft, Dietterich, and M{\"u}ller]{ruff2020unifying}
Ruff, L., Kauffmann, J.~R., Vandermeulen, R.~A., Montavon, G., Samek, W.,
  Kloft, M., Dietterich, T.~G., and M{\"u}ller, K.-R.
\newblock A unifying review of deep and shallow anomaly detection.
\newblock \emph{arXiv preprint arXiv:2009.11732}, 2020.

\bibitem[Sch{\"o}lkopf et~al.(1999)Sch{\"o}lkopf, Williamson, Smola,
  Shawe-Taylor, Platt, et~al.]{scholkopf1999support}
Sch{\"o}lkopf, B., Williamson, R.~C., Smola, A.~J., Shawe-Taylor, J., Platt,
  J.~C., et~al.
\newblock Support vector method for novelty detection.
\newblock In \emph{NIPS}, volume~12, pp.\  582--588. Citeseer, 1999.

\bibitem[Shen et~al.(2020)Shen, Li, and Kwok]{shen2020timeseries}
Shen, L., Li, Z., and Kwok, J.
\newblock Timeseries anomaly detection using temporal hierarchical one-class
  network.
\newblock In \emph{Advances in Neural Information Processing Systems}, 2020.
\newblock URL
  \url{https://proceedings.neurips.cc/paper/2020/file/97e401a02082021fd24957f852e0e475-Paper.pdf}.

\bibitem[Shyu et~al.(2003)Shyu, Chen, Sarinnapakorn, and Chang]{shyu2003novel}
Shyu, M.-L., Chen, S.-C., Sarinnapakorn, K., and Chang, L.
\newblock A novel anomaly detection scheme based on principal component
  classifier.
\newblock Technical report, MIAMI UNIV CORAL GABLES FL DEPT OF ELECTRICAL AND
  COMPUTER ENGINEERING, 2003.

\bibitem[S{\"o}lch et~al.(2016)S{\"o}lch, Bayer, Ludersdorfer, and van~der
  Smagt]{solch2016variational}
S{\"o}lch, M., Bayer, J., Ludersdorfer, M., and van~der Smagt, P.
\newblock Variational inference for on-line anomaly detection in
  high-dimensional time series.
\newblock \emph{arXiv preprint}, 2016.

\bibitem[Su et~al.(2019)Su, Zhao, Niu, Liu, Sun, and Pei]{su2019robust}
Su, Y., Zhao, Y., Niu, C., Liu, R., Sun, W., and Pei, D.
\newblock Robust anomaly detection for multivariate time series through
  stochastic recurrent neural network.
\newblock In \emph{Proceedings of the 25th ACM SIGKDD International Conference
  on Knowledge Discovery \& Data Mining}, pp.\  2828--2837, 2019.
\newblock \doi{10.1145/3292500.3330672}.

\bibitem[Tamkin et~al.(2020)Tamkin, Wu, and Goodman]{tamkin2020viewmaker}
Tamkin, A., Wu, M., and Goodman, N.
\newblock Viewmaker networks: Learning views for unsupervised representation
  learning.
\newblock \emph{arXiv preprint arXiv:2010.07432}, 2020.

\bibitem[Thill et~al.(2020)Thill, Konen, and B{\"a}ck]{thill2020time}
Thill, M., Konen, W., and B{\"a}ck, T.
\newblock Time series encodings with temporal convolutional networks.
\newblock In \emph{International Conference on Bioinspired Methods and Their
  Applications}, pp.\  161--173. Springer, 2020.
\newblock ISBN 978-3-030-63710-1.
\newblock \doi{10.1007/978-3-030-63710-1_13}.

\bibitem[Tschannen et~al.(2019)Tschannen, Djolonga, Rubenstein, Gelly, and
  Lucic]{tschannen2019mutual}
Tschannen, M., Djolonga, J., Rubenstein, P.~K., Gelly, S., and Lucic, M.
\newblock On mutual information maximization for representation learning.
\newblock \emph{arXiv preprint arXiv:1907.13625}, 2019.

\bibitem[Wang et~al.(2019)Wang, Zeng, Liu, Zhu, Yin, Xu, and
  Kloft]{wang2019effective}
Wang, S., Zeng, Y., Liu, X., Zhu, E., Yin, J., Xu, C., and Kloft, M.
\newblock Effective end-to-end unsupervised outlier detection via inlier
  priority of discriminative network.
\newblock In \emph{Advances in Neural Information Processing Systems}, pp.\
  5962--5975, 2019.

\bibitem[Wu \& Keogh(2020)Wu and Keogh]{wu2020current}
Wu, R. and Keogh, E.~J.
\newblock Current time series anomaly detection benchmarks are flawed and are
  creating the illusion of progress.
\newblock \emph{arXiv preprint arXiv:2009.13807}, 2020.

\bibitem[Zhang et~al.(2019)Zhang, Song, Chen, Feng, Lumezanu, Cheng, Ni, Zong,
  Chen, and Chawla]{zhang2019deep}
Zhang, C., Song, D., Chen, Y., Feng, X., Lumezanu, C., Cheng, W., Ni, J., Zong,
  B., Chen, H., and Chawla, N.~V.
\newblock A deep neural network for unsupervised anomaly detection and
  diagnosis in multivariate time series data.
\newblock In \emph{Proceedings of the AAAI Conference on Artificial
  Intelligence}, volume~33, pp.\  1409--1416, 2019.
\newblock \doi{10.1609/aaai.v33i01.33011409}.

\bibitem[Zhao et~al.(2020)Zhao, Wang, Duan, Huang, Cao, Tong, Xu, Bai, Tong,
  and Zhang]{zhao2020multivariate}
Zhao, H., Wang, Y., Duan, J., Huang, C., Cao, D., Tong, Y., Xu, B., Bai, J.,
  Tong, J., and Zhang, Q.
\newblock Multivariate time-series anomaly detection via graph attention
  network.
\newblock In \emph{2020 IEEE International Conference on Data Mining (ICDM)},
  pp.\  841--850. IEEE, 2020.
\newblock \doi{10.1109/ICDM50108.2020.00093}.

\bibitem[Zhou et~al.(2019)Zhou, Liu, Hooi, Cheng, and Ye]{zhou2019beatgan}
Zhou, B., Liu, S., Hooi, B., Cheng, X., and Ye, J.
\newblock Beatgan: Anomalous rhythm detection using adversarially generated
  time series.
\newblock In \emph{IJCAI}, pp.\  4433--4439, 2019.

\bibitem[Zong et~al.(2018)Zong, Song, Min, Cheng, Lumezanu, Cho, and
  Chen]{zong2018deep}
Zong, B., Song, Q., Min, M.~R., Cheng, W., Lumezanu, C., Cho, D., and Chen, H.
\newblock Deep autoencoding gaussian mixture model for unsupervised anomaly
  detection.
\newblock In \emph{International conference on learning representations}, 2018.

\end{thebibliography}
\bibliographystyle{icml2021}

\clearpage
\appendix
\section{Further Implementation Details}
In this section the implementation details for the experiments conducted in the main paper are further elaborated.
These include our method (\gls{lnt}) as well as all baselines that we implemented for comparision.

\subsection{Hardware}
All experiments were run on virtualized hardware with 8 CPU cores of type \textit{Intel(R) Xeon(R) Gold 6150} running at $2.70$ GHz, $32$ GB RAM, and a single \textit{TeslaV100-SXM2} with $32$ GB of gpu memory. Consistently we use \textit{Python 3.9}, \textit{PyTorch} in version $1.8.1$ with \textit{CUDA} in version $11.1$ and \textit{cuDNN} in version $8.0.5$.

\subsection{Hyperparameters}

\paragraph{LNT}
The hyper-parameters for our method were determined by the following procedure.
Starting with the hyper-paramters as reported in \citet{oord2018representation}, the sizes of the embeddings $z_t$ and $c_t$, which also determines the number of memory units in the recurrent part $g_\text{ar}$, and the number of parameters in the convolutional encoder $g_\text{enc}$ are downsized to fit the complexity and amount of data in the other datasets. To find a well generalizing setup, a hold-out validation set (split from the training data) was used. For Libri-Speech we considered the hyper-parameters as optimal and didn't change them. As a rule of thump, the sequence length for training and the width of the strided temporal convolutions were always chosen in a way such that the number of recurrent steps $g_\text{ar}$ takes matches with the setup ($=128$) in \citet{oord2018representation}.

\gls{lnt} is trained for $100$ epochs, respectively $500$ epochs on \gls{swat} and \gls{wadi}, with learning rate $2 \cdot 10^{-4}$, batch size $32$ and $\lambda = 10^{-3}$.

\subsection{Baselines in LibriSpeech Experiments}

The following hyperparameter setups are used for the experiments conducted with synthetic anomalies in LibriSpeech data.

\paragraph{LSTM}
Here, a standard \gls{lstm} network with $2$ layers and $256$ hidden units each was chosen. With this setup the number of hidden units aligns with the \gls{lnt} setup and the multiple layers should account for the missing encoder structure in \gls{lstm}. It is trained until convergence, which took approximately $100$ epochs, with batch size $32$, learning rate $2\cdot10^{-4}$ and a dropout of $0.3$.

\paragraph{THOC}
Here, the Implementation was kindly provided by the authors. We used a smaller sub-sequence length of $1024$ for training due to the high memory load of the model. Predictions at test time are stitched together to align with the longer sequence length. The method is trained to fit $3$ layers hierarchical with dilations $(1,2,4)$, $128$ hidden units and $6$ clusters in each layer. The method is trained with learning rate $10^{-3}$ and batch size $32$ and converged after $50$ epochs. 

\section{Design Choices for Model Evaluation}
The following notes should briefly justify the design choices made for the empirical evaluation of the model in the main paper: 

\begin{itemize}
    \item \textbf{Question:} Why is there a different set of methods for each of the dataset? \\
    \textbf{Answer:} The different baseline sets stem from the different lines of work with reported results \citep{shen2020timeseries, li2019mad}. Especially, \citet{shen2020timeseries} is chosen for its variety of different baselines and its clear experimental setup provided with their code. Unfortunately only \citet{deng2021graph} evaluates on the WaDi dataset with the experimental details provided in \citet{li2019mad}.
    
    \item \textbf{Question:} Why are there different evaluation metrics? \\
    \textbf{Answer:} We choose the Receiver Operating Characteristic (ROC) for its additional insights it provides, beyond the performance of a single threshold. In anomaly detection setups simply optimizing for the best $F_1$ score might not be sufficient. Since a low recall could mean that some anomalies, intrusions or attacks are missed - one often strives for a good $F_1$ above a given recall rate. Since this can be highly application dependent, a good anomaly detector should output a well calibrated anomaly scores that works for different choices of threshold. This is measured by ROC. Since previous work reports $F_1$ scores only, we adopted this for tabular results.
    
    
    \item \textbf{Question:} How come GDN has such a high precision on WaDi? \\
    \textbf{Answer:} \citet{deng2021graph} have shown that graph that model relationship among sensors are a strong representation of the local behaviour in a time series (outperforming the dense vector representations in \gls{lnt} that originate from \gls{cpc} as a time series representation learner.) This yields a very high precision in the anomaly detection task by predicting the next time step give the graph about short term history. But \gls{gdn} learns these graphs locally on a sliding window and thus no broader context (slow features) is included. This might explain the lower recall of the method, if some context-dependent anomalies are missed. We hypothesize that more context-depended outliers can be detected, yielding a higher recall, if local transformations that account for contextualized semantics are applied to the learned graphs, outputting latent graph views that can be contrast. Future work may may especially consider such an approach, unifying the best of both approaches.
    
    Additionally the results reported by \citet{deng2021graph} seem to be hard to reproduce (please see the discussion linked in the main text for details).
    
    
\end{itemize}

\end{document}